\documentclass{article}
\usepackage[margin=1.5in]{geometry}
 
\usepackage{times}
\usepackage[english]{babel}

\usepackage{amsmath}
\usepackage{amsfonts}
\usepackage{amssymb}
\usepackage{mathrsfs}

\usepackage{xspace}
\usepackage{graphicx}

\usepackage{enumerate}

\usepackage{color}
\usepackage{mdframed}

\newcommand{\putawayall}[1]{}

\newcommand{  \logic}{\ensuremath{\textsf{LDA} }\xspace}
\newcommand{  \logicepi}{\ensuremath{\textsf{LEA} }\xspace}

\newcommand{\putaway}[1]{}
\newcommand{\ijcaiputaway}[1]{}

\newcommand{\classbelbase}{\mathbf{MAB} }
\newcommand{\consclassbelbase}{\mathbf{CMAB} }

\newcommand{\expbel}[1] {\mathtt{E}_{#1}   }

\newcommand{\impbel}[1] {\mathtt{I}_{#1}  }

\newcommand{\impbelposs}[1] {\widehat{\mathtt{I}}_{#1}  }

 \newcommand{\notmod}{\mathbf{NDM}}
 \newcommand{\quasinotmod}{\mathbf{QNDM}}

\newcommand{\valfunct}{\mathcal{V}}

 \newcommand{\awbase}{\mathcal{A}}
\newcommand{\belbase}{\mathcal{D}}

\newcommand{\belbaseset}{\mathit{B}}

\newcommand{\iconstraint}{\mathit{Cxt}}

\newcommand{\state}{\mathit{V}}
\newcommand{\relstate}[1]{\mathcal{R}_{#1}}

\newcommand{\doxset}{\mathcal{N}}

\newcommand{\suchthat}{:}

\newcommand{\tagLabel}[2]{\tag{\textbf{#1}}\label{#2}}
\renewcommand{\phi}{\varphi}

\newcommand{\et}{\wedge}

\newcommand{\imp}{\rightarrow} 
\newcommand{\eqv}{\leftrightarrow} 

\renewcommand{\phi}{\varphi}

\newcommand{\eqdef}{\ \stackrel{\mathtt{def}}{=} \ }     

\newcommand{\ATM}{\mathit{Atm}}

\newcommand{\PROP}{\mathit{Atm}}

\newcommand{\AGT}{\mathit{Agt}}

\newcommand{\bnf}{::=}

\newcommand{\lang}{ \mathcal{L} }

\newcommand{\langminus}{ \mathcal{L}^{0} }

\newtheorem{theorem}{Theorem}
\newtheorem{proposition}{Proposition}
\newtheorem{example}{Example}
\newtheorem{lemma}{Lemma}
\newtheorem{definition}{Definition}

\newenvironment{proof}{\medskip\noindent \textsc{Proof.}}
{\hspace*{\fill}\nolinebreak[2]\hspace*{\fill}$\blacksquare$\medskip}

\newenvironment{proofsketch}{\medskip\noindent \textsc{Sketch of Proof.}}
{\hspace*{\fill}\nolinebreak[2]\hspace*{\fill}$\blacksquare$\medskip}

\newbox\itembox
\def\itemlistlabel#1{#1\hfill}
\def\itemlist#1{\setbox\itembox=\hbox{#1}%
                \list{}{\labelwidth\wd\itembox
                             \leftmargin\labelwidth
                             \advance\leftmargin by\itemindent
                             \advance\leftmargin by\labelsep
                             \let\makelabel\itemlistlabel}}

\begin{document}

\title{Rethinking Epistemic Logic with Belief Bases}

\author{Emiliano Lorini\\
CNRS-IRIT, Toulouse University, France
}

\date{}

 \maketitle

\begin{abstract}
We introduce
a new semantics
for a logic
of explicit and implicit beliefs
based on the concept
of multi-agent belief base.
Differently from
existing Kripke-style semantics
for epistemic logic
in which
the notions
of possible world
and doxastic/epistemic alternative
are primitive,
in our semantics
they are non-primitive
but are defined from
the concept of
belief base.
We provide a complete axiomatization
and prove  decidability 
for our logic
via a finite model argument. 
We also provide a polynomial embedding
of our logic into Fagin \& Halpern's logic
of general awareness and establish
a complexity result
for our logic via the embedding.

\end{abstract}

\section{Introduction}

Epistemic logic and, more generally,
formal epistemology are the areas at the intersection
between philosophy \cite{Hintikka}, artificial intelligence (AI) \cite{Fagin1995,Meyer1995} and economics \cite{MonginLismont}
devoted to the formal representation
of epistemic attitudes of agents including belief and knowledge.
An important distinction in epistemic logic is between
\emph{explicit belief} and \emph{implicit belief}.
According to
 \cite{LevesqueExplicitBel},
  ``...a sentence is explicitly believed when it is actively held to be true by an agent and implicitly believed when it follows from what is believed'' (p. \!198).
  This distinction is particularly relevant for the design of  resource-bounded
  agents who spend  time
to make inferences
and do not believe all facts that are deducible from their actual beliefs.

The concept
of explicit  belief
 is tightly
connected with
the concept of  \emph{belief base} \cite{NebelBR,Makinson,HanssonJSL,RottBB}. In particular,
an agent's belief
 base, which is not necessarily closed under
  deduction,  includes all facts
 that are explicitly believed by the agent.
 Nonetheless,
 existing logical formalizations
 of explicit and implicit beliefs  \cite{LevesqueExplicitBel,Fag87}
 do not clearly account for this connection.

 The aim of this paper is
 to fill this gap by
 providing a multi-agent logic that
 precisely articulates the distinction between  explicit
belief, as a fact in an agent's belief base, and implicit belief,
as a fact that is deducible
from the agent's explicit  beliefs,
given the agents' common ground.
 The concept of  \emph{common ground}
  \cite{StalnakerCommonGround}
corresponds to the body of information
that the
agents commonly believe to be the case
and that has to be in the deductive
closure of their belief bases.
 The multi-agent aspect
 of the logic
 lies in the fact that it supports
reasoning about agents' high-order beliefs,
 i.e., an agent's explicit (or implicit) belief
 about the explicit (or implicit) belief of another agent.

 Differently from
existing Kripke-style semantics
for epistemic logic
in which
the notions
of possible world
and doxastic/epistemic alternative
are primitive,
in the semantics of our logic
the notion
of doxastic alternative is defined from --- and more generally
grounded on ---
the concept of
belief base.

We believe that an explicit representation
of agents' belief bases is crucial
in order to facilitate
the task of designing intelligent systems
such as robotic agents or conversational agents.
The problem
of extensional semantics for epistemic logic,
whose most representative example
is the Kripkean semantics,
is their being too abstract
and too far from  the agent specification.
More generally, the main limitation of the Kripkean semantics is that it does not say from where doxastic alternatives come from
thereby being ungrounded.\footnote{The need for a grounded semantics for doxastic/epistemic logics has been pointed out by other authors
including \cite{LomuscioRaimondi2015}.}

\label{complexity}

The paper is organized as follows.
In Section \ref{sec:logic}, we present
 the language of our logic of explicit
and implicit beliefs.
Then, in Section \ref{semanticsGen}, we introduce
 a semantics for this language based
on the notion of multi-agent belief base.
We also consider two additional Kripke-style
semantics
  in which the notion
of doxastic alternative is primitive.
These additional semantics
will be useful for proving
completeness and decidability of our logic.
In Section \ref{finitemodelproperty},
we  show that the three semantics
are all equivalent with respect
to the formal language under consideration.
Then, in Section \ref{Axiomatics}, we  provide an
axiomatization for our logic of explicit and implicit belief
and prove that its satisfiability problem
is decidable.
In Section \ref{complexity},
we provide a polynomial embedding
of our logic into the logic of general awareness by  Fagin \& Halpern \cite{Fag87},
a well-known logic in AI and epistemic game theory.
Thanks to this embedding, we will be able to conclude that
the satisfiability problem of our logic
is PSPACE-complete.
After having discussed related work in Section \ref{relwork}, we conclude.

\section{A Language for Explicit and Implicit Beliefs }\label{sec:logic}

$\logic$ (Logic of Doxastic Attitudes) is a logic
for reasoning about
explicit  beliefs
and implicit beliefs
of multiple agents.
Assume
a countably  infinite set  of atomic propositions $\PROP = \{p,q, \ldots \}$
and
 a finite set of agents $\AGT = \{ 1, \ldots, n \}$.

We define the language of the logic $\logic$
in two steps.
We first define the language $\langminus_\logic(\PROP)$
by the following grammar in Backus-Naur Form (BNF):
\begin{center}\begin{tabular}{lcl}
  $\alpha$  & $\bnf$ & $p  \mid \neg\alpha \mid \alpha_1 \wedge \alpha_2  \mid
  \expbel{i} \alpha
                        $
\end{tabular}\end{center}
where $p$ ranges over $\PROP$
and $i$ ranges over $\AGT$.
 $\langminus_\logic(\PROP)$
is the language for representing
explicit beliefs of multiple agents.
The formula $\expbel{i} \alpha$ is read ``agent  $i$ explicitly  (or actually)
believes that $\alpha$ is true''.
In this language,
we can represent high-order explicit beliefs,
i.e., an agent's explicit belief about
another agent's explicit beliefs.

The language $\lang_\logic(\PROP)$,
extends the language $\langminus_\logic(\PROP)$
by modal operators of implicit belief and
 is defined by the following grammar:
\begin{center}\begin{tabular}{lcl}
  $\phi$  & $\bnf$ & $\alpha  \mid \neg\phi \mid \phi_1 \wedge \phi_2  \mid  \impbel{i} \varphi
                        $\
\end{tabular}\end{center}
where $\alpha$ ranges over $\langminus_\logic(\PROP)$.
For notational convenience we write
$\langminus_\logic$  instead
of $\langminus_\logic(\PROP)$
and  $\lang_\logic$  instead
of $\lang_\logic(\PROP)$, when the context is unambiguous.

The other Boolean constructions  $\top$, $\bot$, $\vee$, $\imp$ and $\eqv$ are defined from $\alpha$, $\neg$ and $\et$ in the standard way.

For every formula $\varphi \in \lang_\logic$,
we write $\ATM(\varphi) $
to denote the set of atomic propositions of type $p$ occurring in $\varphi$.
Moreover,
for every set of formulas $X \subseteq \lang$,
we define
$\ATM(X)  = \bigcup_{\varphi \in X}  \ATM(\varphi)  $.

The formula $ \impbel{i}   \varphi$
has to be read
``agent  $i$ implicitly (or potentially)  believes that $\varphi$ is true''.
We define the dual operator $\impbelposs{i}$
as follows:
\begin{align*}
\impbelposs{i}  \varphi \eqdef  \neg \impbel{i} \neg \varphi.
\end{align*}
$ \impbel{i}   \varphi$
has to be read
``$\varphi$ is compatible with agent $i$'s implicit beliefs''.

\section{Formal Semantics }\label{semanticsGen}

In this section,
we present three formal semantics
for the language
of explicit and implicit beliefs
defined above.
In the first semantics,
the notion of
  doxastic alternative
is not primitive but
it is defined from the primitive
concept
of belief base.
The second semantics is a Kripke-style
semantics, based on the concept
of notional doxastic model, in which
an agent's set of doxastic alternatives
coincides with the set of possible worlds
in which the agent's explicit beliefs are true.
The third semantics
is a weaker semantics, based on the concept of
\emph{quasi}-notional doxastic model.
It only requires  that
an agent's set of doxastic alternatives
has to be included in the set of possible worlds
in which the agent's explicit beliefs are true.
At a later stage in the paper,
we will show that three semantics
are equivalent with respect
to the formal language under consideration.

\subsection{Multi-agent belief base semantics  }\label{semanticSect}

We first consider
the semantics
based on the concept
of multi-agent belief base
that is defined as follows.

\begin{definition}[Multi-agent belief base]\label{MAB}
A multi-agent belief base is a tuple $ B = (\belbaseset_1, \ldots, \belbaseset_n,  \state )$
where:
\begin{itemize}
\item for every $i  \in \AGT$, $\belbaseset_i \subseteq \langminus_\logic  $
is agent $i$'s belief base,
\item $ \state \subseteq \PROP  $ is the actual state.

\end{itemize}
\end{definition}

A similar concept
is used in  belief merging \cite{KoniecznyPerez}
in which each agent
is identified with her belief
base.
Our concept of multi-agent belief base
also includes the concept
of actual state, as
the set of true atomics facts.

The sublanguage
$\langminus_\logic(\PROP)$
is interpreted with respect
to multi-agent belief bases, as follows.
\begin{definition}[Satisfaction relation]\label{truthcond1}
Let $ B =\\ (\belbaseset_1, \ldots,  \belbaseset_n,  \state )$ be a multi-agent belief base. Then:
\begin{eqnarray*}
B \models p & \Longleftrightarrow & p \in \state \\
B\models \neg \alpha & \Longleftrightarrow &    B \not \models  \alpha \\
B \models \alpha_1 \wedge \alpha_2 & \Longleftrightarrow &    B  \models \alpha_1  \text{ and }     B \models \alpha_2 \\
B \models \expbel{i} \alpha   & \Longleftrightarrow & \alpha \in  \belbaseset_i
\end{eqnarray*}

\end{definition}

The following definition introduces the
concept
of doxastic alternative.

\begin{definition}[Doxastic alternatives]
Let $ B =(\belbaseset_1, \ldots, \\ \belbaseset_n, \state )$ and $ B' = (\belbaseset_1', \ldots,  \belbaseset_n',  \state' )$ be two multi-agent belief bases.
Then,  $B \relstate{i} B'$ if and only if, for every $\alpha \in \belbaseset_i $, $B' \models \alpha$.

\end{definition}
$B \relstate{i} B'$ means
that
$B'$
is a doxastic alternative for agent $i$ at $B$
(i.e., at $B$
agent $i$ considers $B'$ possible).
 The idea of the previous definition is that
$B'$
is a doxastic alternative for agent $i$ at $B$
if and only if,
$B'$ satisfies all
facts that agent $i$ explicitly believes
at $B$.

A multi-agent belief model (MAB)
is defined to be a
multi-agent belief base
supplemented with
a
set of multi-agent belief bases,
called \emph{context}.
The latter
includes
 all multi-agent belief bases that are compatible with
the agents' common ground \cite{StalnakerCommonGround},
i.e., the body of information that the agents commonly believe to be the case.
\begin{definition}[Multi-agent belief model]\label{MAM}
A multi-agent belief model (MAB)
is a pair $ (B,\iconstraint)$,
where
 $B $
is a multi-agent belief base
and
$\iconstraint$
is a set of multi-agent belief bases.
The class of  MABs is denoted by $\classbelbase$.
\end{definition}
Note that in the previous
definition
we do not require $ B \in \iconstraint$.
Let us illustrate the concept of MAB with the aid of an example.
\begin{example}
Let $\AGT = \{ 1, 2 \}$
and $\{p,q  \} \subseteq \PROP  $.
Moreover, let $    (\belbaseset_1,   \belbaseset_2, \state )                    $ be such that:
\begin{align*}
\belbaseset_1 & = \{p,  \expbel{2} p \},\\
\belbaseset_2 & = \{p \},\\
\state & =  \{p, q\}.
\end{align*}
Suppose that
the agents have in their common ground the fact $p \rightarrow q $. In other words, they
commonly believe that $p$ implies $q$.
This means that:
\begin{align*}
\iconstraint = \{       B'     \suchthat   B' \models      p \rightarrow q      \}.
\end{align*}
\end{example}

The following definition
generalizes
Definition
\ref{truthcond1}
to the full language $\lang_\logic(\PROP)$.
Its formulas
are interpreted with
respect
to MABs.
(Boolean cases are omitted, as they are defined in the usual way.)
\begin{definition}[Satisfaction relation (cont.)]\label{truthcond2}
Let $ (B,\iconstraint)  \in \classbelbase$. Then:
\begin{eqnarray*}
 (B,\iconstraint) \models \alpha & \Longleftrightarrow & B \models \alpha \\
 (B,\iconstraint) \models \impbel{i} \varphi & \Longleftrightarrow & \forall B' \in  \iconstraint  : \text{if } B \relstate{i} B' \text{ then}\\
 && ( B' , \iconstraint) \models \varphi
\end{eqnarray*}

\end{definition}

Let us go back to the example.

\begin{example}
It is to check that the following holds:
\begin{align*}
(B,\iconstraint )\models  \impbel{1} (p \wedge q) \wedge  \impbel{2} (p \wedge q)  \wedge \impbel{1}  \impbel{2} (p \wedge q).
\end{align*}
Indeed, we have:
\begin{align*}
\relstate{1} (B) \cap \iconstraint 		&= \{       B'     \suchthat   B' \models  p \wedge  \expbel{2} p   \wedge (p \rightarrow q)    \},\\
\relstate{2} (B) \cap \iconstraint &= \{       B'     \suchthat   B' \models  p   \wedge (p \rightarrow q)    \},
\end{align*}
and, consequently,
\begin{align*}
\big(\relstate{1} \circ \relstate{2} (B) \big) \cap \iconstraint = \{       B'     \suchthat   B' \models  p    \wedge (p \rightarrow q)    \},
\end{align*}
where $\circ$
is the composition operation between binary relations
and
$\relstate{i} (B) = \{B' \suchthat B \relstate{i} B'\}$.
\end{example}

Here, we consider consistent MABs
that guarantee consistency of the agents' belief bases.
Specifically:
\begin{definition}[Consistent MAB]\label{CMAM}
 $ (B,\iconstraint)$ is a consistent MAB (CMAB)
if and only if, for every $B ' \in \iconstraint \cup \{B\}$,
there exists $B'' \in \iconstraint$
such that $B ' \relstate{i} B''$.
The class of  CMABs is denoted by $\consclassbelbase$.
\end{definition}

Let $\varphi \in \lang$,
we say that $\varphi$
is valid
for the class  of CMABs
if and only if, for every  $(B , \iconstraint) \in \consclassbelbase $,
we have $ (B , \iconstraint) \models \varphi $.
We say that
$\varphi$
is satisfiable
for the the class  of  CMABs
if and only if $\neg \varphi $
is not
valid
for the the class of CMABs.

\subsection{Notional doxastic model semantics }\label{semanticSect}

Let us now
consider
the semantics
for
 $\logic$
 based on the concept of
 notional doxastic model (NDM).
 It
 is defined
 in the next Definition \ref{modeldef},
 together with
 the satisfaction
 relation for the formulas
 of the language $\lang_\logic(\PROP)$.

\begin{definition}[Doxastic model]\label{modeldef}
A notional doxastic model (NDM) is a tuple $ M = (W , \belbase, \doxset, \valfunct )$
where:
\begin{itemize}
\item $W$ is a set of worlds,

\item $\belbase : \AGT \times W \longrightarrow 2^{\langminus_\logic}$ is a doxastic function,

\item $ \doxset :  \AGT \times W  \longrightarrow 2^{W}$
is a notional function, and

\item $ \valfunct : \PROP \longrightarrow 2^{W}$ is a valuation function,
\end{itemize}
and that satisfies the following conditions
for all
$i \in \AGT$ and
$ w \in W$:
\begin{itemlist}{C5}
\item[(C1)]    $   \doxset(i,w)  = \bigcap_{\alpha \in  \belbase(i,w) }  ||\alpha ||_M$, and
\item[(C2)]    there exists $v \in W$ such that $ v\in   \doxset(i,w) $,
\end{itemlist}
with:
\begin{eqnarray*}
 (M,w) \models p & \Longleftrightarrow & w \in  \valfunct(p) \\
   (M,w) \models \neg \varphi & \Longleftrightarrow &    (M,w) \not \models  \varphi \\
   (M,w) \models \varphi \wedge \psi & \Longleftrightarrow &   (M,w) \models \varphi   \text{ and }    (M,w)\models \psi \\
    (M,w) \models \expbel{i} \alpha   & \Longleftrightarrow & \alpha \in  \belbase (i,w) \\
(M,w) \models \impbel{i} \varphi & \Longleftrightarrow & \forall v \in  \doxset(i,w)  :  (M,v) \models \varphi
\end{eqnarray*}
and
\begin{align*}
 ||\alpha ||_M = \{ v \in W \suchthat (M,v) \models \alpha \}.
\end{align*}
 \end{definition}
 The class of notional doxastic models is denoted by $\notmod$.

 We say that a NDM $ M = (W , \awbase, \belbase, \doxset, \valfunct )$
is \emph{finite}
if and only if
$W$,
$\belbase (i,w)$
and $\valfunct^{-1} (w)$
are finite sets
for every $i \in \AGT$
and for every $w\in W$,
where $\valfunct^{-1}$
is the inverse function of $\valfunct$.

For every agent $i$
and world $w$,
$\belbase (i,w) $ denotes
  agent $i$'s set
  of explicit beliefs
  at $w$.

  The set
$ \doxset(i,w)$,
used in the interpretation
of the implicit belief operator
$\impbel{i}$,
is called agent $i$'s set of \emph{notional} worlds at world $w$.
The term `notional' is taken from \cite{DennettIntStance,Dennett2} (see, also,  \cite{KONOLIGE}):
an agent's notional world
is  \emph{a world
in which all the agent's explicit beliefs are true}.
This idea is clearly expressed
by the Condition C1.
According to the
Condition C2, an agent's set of notional worlds
must be non-empty. This guarantees
consistency
of the agent's implicit beliefs.

Let $\varphi \in \lang$,
we say that $\varphi$
is valid
for the the class  of NDMs
if and only if, for every $M = (W , \awbase, \belbase, \doxset, \valfunct ) \in \notmod $ and for every $w \in W$,
we have $ (M,w) \models \varphi $.
We say that
$\varphi$
is satisfiable
for the the class  of  NDMs
if and only if $\neg \varphi $
is not
valid
for the the class of NDMs.

\subsection{Quasi-model semantics }

In this section we provide an alternative
semantics for the logic $\logic$
based on a more general class of models, called
quasi-notional doxastic models (quasi-NDMs).
This semantics will be fundamental
for proving completeness of $\logic$.

\begin{definition}[Quasi-notional doxastic model]\label{QNDM}
A quasi-notional doxastic model (quasi-NDM) is a tuple $ M = (W , \belbase, \doxset  ,\valfunct )$
where
$W ,  \belbase, \doxset$ and $\valfunct $
are as in Definition
 \ref{modeldef}
 except that Condition C1 is replaced by the following weaker condition,
 for all
$i \in \AGT$ and
$ w \in W$:
\begin{itemlist}{C8}
\item[(C1$^*$)]     $   \doxset(i,w)  \subseteq \bigcap_{\alpha \in  \belbase(i,w) }  ||\alpha ||_M$.
\end{itemlist}
 \end{definition}
  The class of  quasi-notional doxastic models is denoted by $\quasinotmod $.
 Truth conditions of formulas in $\lang$
 relative to this class
 are the same as
  truth conditions of formulas in $\lang$
 relative to the class $\notmod $.
 Validity and satisfiability of
 a $\logic$ formula
 $\varphi$
 for the the class of quasi-NDMs
 are defined in the usual way.

 As for NDMs, we say that a quasi-NDM $ M = (W , \belbase, \doxset, \valfunct )$
is \emph{finite}
if and only if
$W$, $\belbase (i,w)$
and  $\valfunct^{-1} (w)$
are finite sets
for every $i \in \AGT$
and for every $w\in W$.

\section{Equivalences between semantics}\label{finitemodelproperty}

The present section is devoted to present  equivalences
between the different semantics for the language $\lang_\logic(\PROP)$.
The results of the section are summarized in Figure \ref{fig:equivfigure}.
%
%
%
%
%

\begin{figure}[h]
\centering
\includegraphics[width=0.47\textwidth]{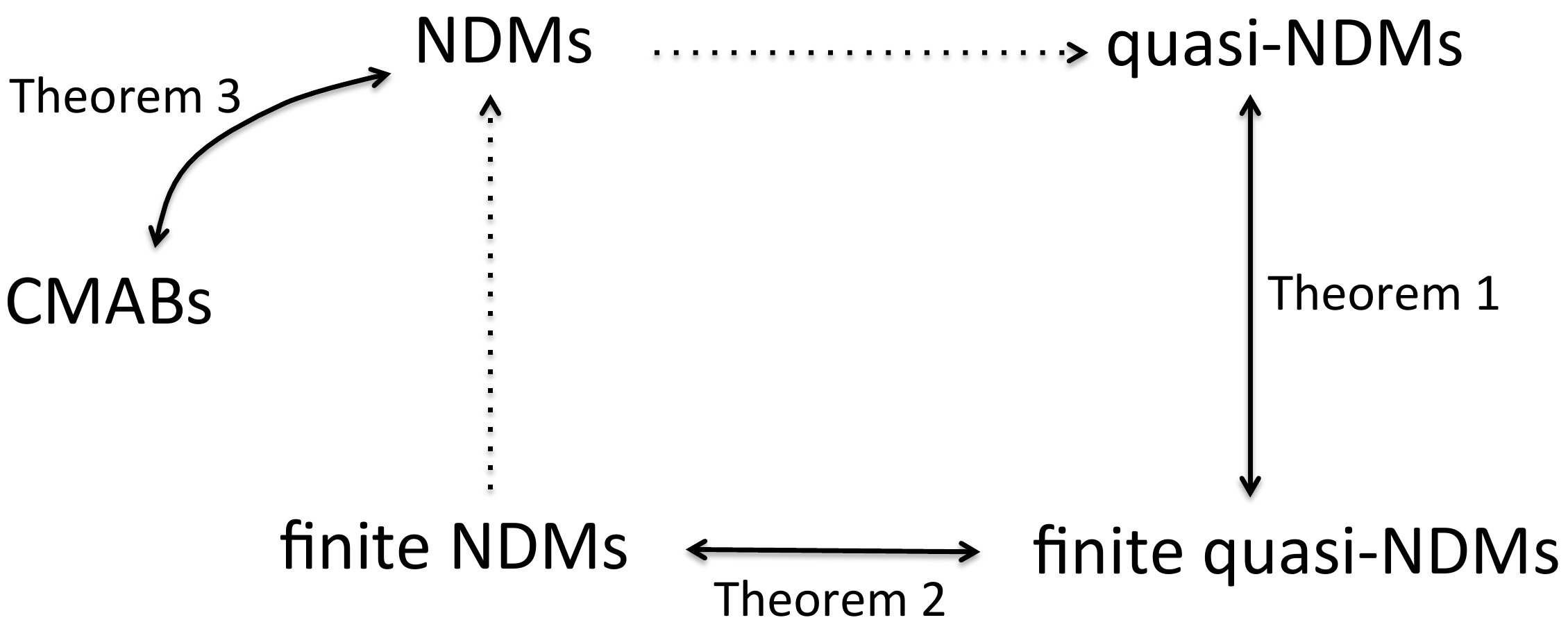}
\caption{Relations between semantics. An arrow means
that satisfiability relative to the first class of structures
 implies satisfiability relative to the second class of structures. Dotted arrows
 denote relations that follow straightforwardly
 given the inclusion between classes of structures.  }
\label{fig:equivfigure}
\end{figure}

The figure highlights that the five semantics
for the language $\lang_\logic(\PROP)$ defined in the previous section
are all equivalent, as from every node in the graph
we can reach all other nodes.

\paragraph{Equivalence between quasi-NDMs and finite quasi-NDMs}

We use a filtration argument
to show that if a
formula $\varphi$
of the language $ \lang$
is true in a (possibly infinite) quasi-NDM
then it is true in a finite quasi-NDM.

Let $ M = (W ,  \belbase, \doxset, \valfunct )$ be a (possibly
infinite) quasi-NDM
and let $\Sigma \subseteq \lang_\logic$
be an arbitrary finite set of formulas which is closed under subformulas. (Cf. Definition 2.35 in \cite{Bla01} for
a definition of subformulas closed set of formulas.)
Let the equivalence relation $\equiv_\Sigma$ on $W$ be defined as follows. For all $w,v \in W$:
\begin{align*}
w \equiv_\Sigma v \text{ iff } \forall \varphi \in \Sigma : (M,w) \models \varphi \text{ iff } (M,v) \models \varphi.
\end{align*}
Let $|w |_\Sigma$
be the equivalence class of the world $w$ with respect to the equivalence relation $\equiv_\Sigma$.

We define $W_\Sigma$ to be the filtrated set of worlds with respect to $\Sigma$:
\begin{align*}
W_\Sigma = \{ |w |_\Sigma \suchthat w \in W \}.
\end{align*}
Clearly, $W_\Sigma$
is a finite set.

Let us define the filtrated valuation function $\valfunct_\Sigma $. For every $p \in \PROP$,
we define:
\begin{align*}
\valfunct_\Sigma (p) & = \{  |w |_\Sigma \suchthat (M,w) \models p  \}  & \text{if } p \in \PROP(\Sigma)  \\
\valfunct_\Sigma (p)  & = \emptyset &  \text{otherwise}
\end{align*}

 The next step in the construction consists in defining the filtrated
 doxastic function. For every $i \in \AGT$
and for every $ |w |_\Sigma  \in W_\Sigma$, we define:
\begin{align*}
\belbase_\Sigma  (i, |w |_\Sigma)  &= \big(\bigcap_{w \in  |w |_\Sigma }\belbase(i,w) \big)\cap \Sigma .
\end{align*}

Finally,
for every $i \in \AGT$
and for every $ |w |_\Sigma \in W_\Sigma$,
we define
agent $i$'s set of notional worlds
at $ |w |_\Sigma$ as follows:
\begin{align*}
\doxset_\Sigma(i, |w |_\Sigma)  = \{ |v |_\Sigma \suchthat v \in  \doxset(i,w)  \}.
\end{align*}


We call the model
$ M_\Sigma = (W_\Sigma , \awbase_\Sigma, \belbase_\Sigma, \doxset_\Sigma, \valfunct_\Sigma )$
the filtration of $M$
under $\Sigma$.

We can state the following filtration lemma.
\begin{lemma}\label{theoremFiltration}
Let $\varphi \in  \Sigma$
and let $w \in W$.
Then,
$(M,w) \models \varphi$
if and only if $(M_\Sigma, |w |_\Sigma) \models \varphi$.
\end{lemma}

\begin{proof}
The proof is by induction on the structure of $\varphi$.
For the ease of exposition,
we prove our result
for the language $\lang$
in which
the ``diamond'' operator
$ \impbelposs{i} $
is taken as primitive and
the ``box'' operator
$ \impbel{i} $
is defined from it.
Since
the two operators are inter-definable,
this does not affect the validity
of our result.

The case $\varphi = p $
is immediate from the definition of $\valfunct_\Sigma$.
The boolean cases $\varphi = \neg \psi $
and  $\varphi = \psi_1 \wedge \psi_2$
follow straightforwardly from the fact that $\Sigma$
is closed under subformulas. This allows us to apply the induction hypothesis.

Let us prove the case  $\varphi =  \expbel{i} \alpha $.

($\Rightarrow$)
Suppose $(M,w) \models  \expbel{i} \alpha  $
with $ \expbel{i} \alpha  \in \Sigma$.
Thus, $ \alpha \in \belbase(i,w)$. Hence,
by definition of $\belbase_\Sigma(i, |w |_\Sigma)$,
 the fact that $\Sigma$
is closed under subformulas
and
the fact that
if $\expbel{i} \alpha  \in \Sigma$
and $ \alpha \in \belbase(i,w)$ then 
 $\alpha \in \bigcap_{w \in  |w |_\Sigma }\belbase(i,w) $, we have $ \alpha  \in \belbase_\Sigma(i, |w |_\Sigma)$.
It follows that $(M_\Sigma, |w |_\Sigma) \models \expbel{i} \alpha $.

($\Leftarrow$) For the other direction,
suppose $(M_\Sigma, |w |_\Sigma) \models \expbel{i} \alpha $
with $ \expbel{i} \alpha  \in \Sigma$.
Thus, $ \alpha \in \belbase_\Sigma(i, |w |_\Sigma)$.
Hence,
by definition of $\belbase_\Sigma(i, |w |_\Sigma)$, $ \alpha  \in \belbase(i,w)$.

%

Let us conclude the proof for the case $\varphi =  \impbelposs{i} \psi $.
It is easy to check that
$\doxset_\Sigma$ gives rise to the smallest filtration and that the following two properties hold
for all $w,v \in W$
and for all $i \in \AGT$:

\begin{itemize}

\item[(i)] if $v \in \doxset(i,w)$ then $ |v |_\Sigma \in \doxset_\Sigma(i, |w |_\Sigma)  $, and

\item[(ii)] if $ |v |_\Sigma \in \doxset_\Sigma(i, |w |_\Sigma)  $
then for all $ \impbelposs{i} \varphi \in \Sigma $, if $M, v \models   \varphi$ then $M, w \models \impbelposs{i}  \varphi$.

\end{itemize}

($\Rightarrow$)
Suppose $(M,w) \models  \impbelposs{i}  \psi   $
with $  \impbelposs{i}  \psi   \in \Sigma$.
Thus, there exists $v  \in \doxset(i,w)$ such that $(M,v) \models \psi$.
By the previous item (i),
$ |v |_\Sigma  \in \doxset_\Sigma(i, |w |_\Sigma)$.
Since $\Sigma$
is closed under subformulas,
we have $\psi \in \Sigma$.
Thus, by the induction hypothesis, $(M_\Sigma, |v |_\Sigma) \models \psi$.
It follows that $(M_\Sigma, |w |_\Sigma) \models  \impbelposs{i}  \psi$.

($\Leftarrow$) For the other direction,
suppose $(M_\Sigma, |w |_\Sigma) \models  \impbelposs{i}  \psi  $
with $  \impbelposs{i}  \psi   \in \Sigma$.
Thus,  there exists $ |v |_\Sigma  \in \doxset_\Sigma(i, |w |_\Sigma)$, such that $(M_\Sigma, |v |_\Sigma) \models \psi$.
Since $\Sigma$
is closed under subformulas,
by the induction  hypothesis,
we have  $(M, v) \models \psi$.
By the item (ii) above, it follows that $(M, w) \models  \impbelposs{i}  \psi$.
\end{proof}

The next step  consists in proving that
$M_\Sigma$
is the right model construction.

\begin{proposition}\label{goodcon}
The tuple
$ M_\Sigma = (W_\Sigma , \awbase_\Sigma, \belbase_\Sigma, \doxset_\Sigma, \valfunct_\Sigma )$
is a finite quasi-NDM.
\end{proposition}
\begin{proof}
Clearly, $M_\Sigma$
is finite.
Moreover, it is easy to verify that it satisfies the Condition C2 in Definition \ref{modeldef}.
We are going to prove that it satisfies the Condition C1$^*$
in Definition \ref{QNDM}.

By Lemma \ref{theoremFiltration},
if $\alpha \in  \big(\bigcap_{w \in  |w |_\Sigma }\belbase(i,w) \big)\cap \Sigma$
then $ ||\alpha ||_{M_\Sigma}  =  \{    |v |_\Sigma \suchthat v \in  ||\alpha ||_{M}    \} $.
Moreover, as $M$ is a quasi-NDM,
we have
\begin{align*}
 \doxset(i,w)  \subseteq \bigcap_{\alpha \in  \belbase(i,w)  }   ||\alpha ||_{M} \subseteq
\bigcap_{\alpha \in  \big(\bigcap_{w \in  |w |_\Sigma }\belbase(i,w) \big)\cap \Sigma }   ||\alpha ||_{M}  .
\end{align*}
Hence, by definitions of $\doxset_\Sigma(i, |w |_\Sigma) $ and $\belbase_\Sigma$,
\begin{align*}
\doxset_\Sigma(i, |w |_\Sigma)    \subseteq
  \bigcap_{\alpha \in  \belbase_\Sigma  (i, |w |_\Sigma)   }   ||\alpha ||_{M_\Sigma}.
\end{align*}
\end{proof}

The following is
our first result
about equivalence
between
the semantics
in terms
of quasi-NDMs
and the semantics
in terms
of finite quasi-NDMS.

\begin{theorem}\label{finitemodel}
Let $\varphi \in \lang$.
Then, if $\varphi $
is satisfiable
for the class
of quasi-NDMs,
if and only if it
is satisfiable
for the class
of finite quasi-NDMs.
\end{theorem}
\begin{proof}
The right-to-left direction is obvious.
As for the left-to-right direction,
let  $M$
be a possibly infinite quasi-NDM and
let $w$ be a world in $M$
such that $(M,w) \models \varphi$.
Moreover, let $\mathit{sub}(\varphi)$
be the set of subformulas of $\varphi$.
 Then, by Lemma \ref{theoremFiltration}
and Proposition \ref{goodcon}, $(M_{\mathit{sub}(\varphi)}, |w |_{\mathit{sub}(\varphi)} ) \models \varphi $
and $M_{\mathit{sub}(\varphi)}$
is a finite quasi-NDM.
\end{proof}


\paragraph{Equivalence between finite NDMs and finite quasi-NDMs}

As the following theorem highlights,
the  $\logic$ semantics
in terms of finite NDMs
and the  $\logic$ semantics
in terms  of  finite quasi-NDMs
are equivalent.

\begin{theorem}\label{eqsemantics}
Let
 $\varphi \in \lang$.
 Then, $\varphi$
 is satisfiable for  the class
 of finite NDMs if and only if
 $\varphi$
 is satisfiable for the class
 of finite quasi-NDMs.
\end{theorem}
\begin{proof}
The left-to-right direction is obvious.
We are going to prove the right-to-left direction.

Let $ M = (W , \belbase, \doxset  ,\valfunct )$ be a
finite
quasi-NDM that satisfies $\varphi$, i.e.,
there exists $w \in W$ such that $(M,w) \models \varphi$.
Let
\begin{align*}
\mathcal{T}(M) = \cup_{w \in W, i \in \AGT} \ATM(\belbase(i,w) )
\end{align*}
be the \emph{terminology}
of model $M$ including all atomic propositions that are in the explicit beliefs
of some  agent at some world in $M$.
Since $M$ is finite, $\mathcal{T}(M) $
is finite too.

Let us introduce an injective function:
\begin{align*}
f: \AGT \times W \longrightarrow  \PROP \setminus ( \mathcal{T}(M)  \cup  \PROP(\varphi)  )
\end{align*}
which assigns an identifier to every agent in $\AGT$ and world in $W $.
The fact that $\PROP$
is infinite while $W$, $ \mathcal{T}(M) $ and $ \PROP(\varphi) $
are finite guarantees that such an injection  exists.

The next step consists in defining the new model
 $ M' = (W' , \belbase', \doxset'  ,\valfunct' )$
 with $W' = W$, $\doxset' = \doxset$
 and where $  \belbase'$ and $\valfunct'$
 are defined as follows.

 For every $i \in \AGT$
and for every $ w  \in W$:
\begin{align*}
\belbase'  (i, w)  &= \belbase(i,w) \cup \{ f(i,w)  \}.
\end{align*}

Moreover, for every $p \in \PROP$:
\begin{align*}
\valfunct' (p) & = \valfunct(p)  & \text{ if }p \in \mathcal{T}(M) \cup  \PROP(\varphi), \\
\valfunct' (p) & =  \doxset(i,w)   & \text{ if }p = f(i,w), \\
\valfunct' (p) & =  \emptyset   & \text{ otherwise.}
\end{align*}

It is easy  to verify that
$   \doxset'(i,w)  = \bigcap_{\alpha \in  \belbase'(i,w) }  ||\alpha ||_{ M' }$
 for all $i \in \AGT$
and for all $ w  \in W'$
and, more generally, that
 $M'$ is a finite NDM.

  By induction on the structure of $\varphi$,
we prove that, for all $w \in W$,
  ``$(M, w) \models \varphi$
  iff $(M',w) \models \varphi$''.

 The case $\varphi = p $
is immediate from the definition of $\valfunct'$.
By the induction hypothesis, we can prove
the boolean cases $\varphi = \neg \psi $
and  $\varphi = \psi_1 \wedge \psi_2$
in a straightforward manner.

Let us prove the case  $\varphi =  \expbel{i} \alpha $.

($\Rightarrow$)
Suppose $(M,w) \models  \expbel{i} \alpha$. Then,
we have $\alpha \in \belbase (i, w) $. Hence, by the definition
of $\belbase'$,  $\alpha \in \belbase' (i, w) $. Thus, $(M',w) \models  \expbel{i} \alpha$.

($\Leftarrow$)
Suppose $(M',w) \models  \expbel{i} \alpha$. Then, we
 have $\alpha \in \belbase' (i, w) $. The definition
of $\belbase'$ ensures that $\alpha \neq  f(i,w)  $, since $f(i,w) \not \in \PROP(\expbel{i} \alpha) $. Thus, $\alpha \in \belbase (i, w) $ and, consequently, $(M,w) \models  \expbel{i} \alpha$.

Let us prove the case  $\varphi =  \impbel{i} \psi $.
 $(M,w) \models   \impbel{i} \psi $ means that
 $(M,v) \models \psi$
for all $ v \in  \doxset(i,w)   $.
By induction hypothesis  and the fact that $ \doxset(i,w) =  \doxset'(i,w) $, the latter
 is equivalent to
  $(M',v) \models \psi$
for all $ v \in  \doxset'(i,w)   $.
The latter means that $(M',w) \models   \impbel{i} \psi $.

Since $M$ satisfies $\varphi$
and
  ``$(M, w) \models \varphi$
  iff $(M',w) \models \varphi$''
   for all $w \in W$,
   $M'$ satisfies $\varphi$ as well.
\end{proof}

\paragraph{Equivalence between CMABs and  NDMs}

Our third equivalence result
is between  CMABs and NDMs.

\begin{theorem}\label{eqsemantics3}
Let
 $\varphi \in \lang$.
 Then, $\varphi$
 is satisfiable for  the class
 of CMABs if and only if
 $\varphi$
 is satisfiable for the class
 of NDMs.
\end{theorem}

\begin{proof}
We first prove the left-to-right direction.
Let $(B, \iconstraint)$
be a CMAB with
$ B = ( \belbaseset_1, \ldots, \belbaseset_n,  \state )$
and
such that $(B, \iconstraint) \models \varphi$.
We define the structure
 $ M = (W , \belbase,\doxset  ,\valfunct )$ as follows:
\begin{itemize}
\item $W = \{ w_{ B'}  \suchthat B' \in  \iconstraint  \cup \{B\}\}$,

 \item for every $i \in \AGT$ and
for every $w_{ B'} \in W$,
if $ B' = ( \belbaseset_1', \ldots, \belbaseset_n',  \state' )$ then
$ \belbase(i,w_{ B'})  = B_i'$,

\item for every $i \in \AGT$ and
for every $w_{ B'} \in W$, $  \doxset(i,w_{ B'})  = \bigcap_{\alpha \in  \belbase(i,w_{ B'}) }  \{  w_{ B''}\in W \suchthat B''\models \alpha \} $,

\item for every $p \in \ATM$, $\valfunct (p) = \{ w_{ B'}  \in W \suchthat B' \models p \} $.
\end{itemize}
One can show that $M$ so defined is a NDM.
Moreover, by induction on the structure of $\varphi$,
one can prove that,
for all $w_{ B'} \in W$,
$M, w_{ B'} \models \varphi$
iff $(B', \iconstraint)\models \varphi$.
Thus,
$(M, w_B ) \models \varphi$.


We now prove the right-to-left direction.
Let  $ M = (W , \belbase,\doxset  ,\valfunct )$
be a  NDM and let $w$ be a world in $W$
such that $(M,w ) \models \varphi$.
Let us say that a  NDM
$ M = (W , \belbase,\doxset  ,\valfunct )$
is non-redundant iff there are no $w,v \in W$
such that $\valfunct^{-1}(w) = \valfunct^{-1}(v)$,
and, for all $i \in \AGT$, $\belbase (i,w) = \belbase (i,v)$.
It is straightforward to show that
if $\varphi$
 is satisfiable for  the class
 of NDMs then
 $\varphi$
 is satisfiable for the class
 of non-redundant NDMs. Thus,
 from the initial model $M$, we can find a non-redundant NDM
$ M'= (W', \belbase',\doxset'  ,\valfunct' )$
and $v\in W'$
such that $(M',v) \models \varphi$.
For every $u \in W'$
we define $ B^u = ( \belbaseset_1^u, \ldots, \belbaseset_n^u,  \state^u )$ such that
 $\belbaseset_i^u =  \belbase(i,u) $  for every $i \in \AGT$
and $ \state^u =  \valfunct^{-1} (u)$. Moreover, we define the context
 $\iconstraint = \{B^u \suchthat u \in W'     \}$.
 One can show that, for every $B^u \in \iconstraint$, $(B^u, \iconstraint)$ is a CMAB.
The fact that $M'$
is non-redundant is essential to guarantee that there is a one-to-one correspondence between
$W'$ and $\iconstraint $.
By induction on the structure of $\varphi$,
one can prove that,
for all $B^u \in \iconstraint $,
$(B^u, \iconstraint)\models \varphi$
iff $M', u\models \varphi$.
Thus,
$B^v \models \varphi$.

\end{proof}

\section{Axiomatics and decidability }\label{Axiomatics}

This section is devoted to provide
an axiomatization and a decidability result for $\logic$. To this aim,
we first provide a formal definition
of this logic.

\begin{definition}
We define $\logic $  to be the extension of classical propositional logic given by the following axioms and rule of inference:
\begin{align}
& ( \impbel{i} \varphi \wedge  \impbel{i} (\varphi \rightarrow  \psi) ) \rightarrow \impbel{i} \psi  \tagLabel{K$_{\impbel{i}}$}{ax:KK}\\
&  \neg (\impbel{i} \varphi \wedge \impbel{i} \neg \varphi ) \tagLabel{D$_{\impbel{i}}$}{ax:KD}\\
& \expbel{i} \alpha   \rightarrow \impbel{i} \alpha  \tagLabel{Int$_{\expbel{i},\impbel{i}}$}{ax:Int1}\\
 &     \frac{   \varphi }{   \impbel{i}  \varphi    }
                                            \tagLabel{Nec$_{\impbel{i}  }$}{ax:NecK}
\end{align}
\end{definition}
We denote that $\varphi $
is derivable in $\logic$
by $\vdash_\logic \varphi$. We say that $\varphi$
is $\logic$-consistent if $\not \vdash_\logic \neg \varphi$.

The logic $\logic $
includes the principles of
system KD for the implicit belief operator $ \impbel{i} $
as well as an axiom \ref{ax:Int1} relating explicit belief
with implicit belief. Note that there is no consensus in the literature about introspection for implicit belief. For instance, in his seminal work on the logics of knowledge and belief \cite{Hintikka},
Hintikka only assumed positive introspection for belief (Axiom 4) and rejected negative introspection (Axiom 5). Other logicians such as \cite{Jones2015} have argued against the use of both positive and negative introspection axioms for belief.
Nonetheless, all approaches unanimously assume that a reasonable notion of implicit belief should satisfy Axioms K and D. In this sense, system KD can be conceived as the minimal logic of implicit belief. On this point, see
\cite{BanerjeeDubois2014}.

To prove our main completeness result,
we first prove a theorem
about soundness
and completeness of $\logic$
for the class of quasi-NDMs.

\paragraph{Soundness and completeness for quasi-NDMs}
To prove completeness of $\logic$
for the class of quasi-NDMs, we use a canonical model argument.

We consider  maximally $\logic$-consistent sets  of formulas in  $\lang $
(MCSs). The following proposition
specifies some usual properties of MCSs.
 \begin{proposition}  \label{propmcs}
Let $ \Gamma $  be a MCS and let $\varphi, \psi \in \lang$.  Then:
\begin{itemize}

\item if $\varphi, \varphi \rightarrow \psi \in \Gamma$
then $\psi \in \Gamma$;

\item $ \varphi \in \Gamma$
or $\neg \varphi \in \Gamma$;

\item $\varphi \vee \psi  \in \Gamma$
iff $ \varphi \in \Gamma$
or $\psi \in \Gamma$.

\end{itemize}

  \end{proposition}

     The following is
the Lindenbaum's lemma for our logic.
Its proof is standard (cf. Lemma 4.17 in \cite{Bla01})
and we omit it.

\begin{lemma}  \label{lindenb}
Let $ \Delta $  be
a $\logic$-consistent set of formulas.
Then, there exists a  MCS $\Gamma$ such that $\Delta \subseteq \Gamma$.
  \end{lemma}

  Let the canonical quasi-NDM model be the tuple
  $ M = (W^c , \belbase^c, \doxset^c, \valfunct^c)$ such that:
 \begin{itemize}
  \item $W^c$ is set of all MCSs;

      \item for all $w \in  W^c$, for all $i \in  \AGT$ and for all  $\alpha \in \langminus_\logic $,
    $ \alpha \in \belbase^c(i,w) $ iff $\expbel{i} \alpha \in w$;

    \item for all $w,v \in  W^c$ and for all $i \in  \AGT$,
    $ v \in \doxset^c(i,w) $ iff, for all $\varphi \in \lang $,
    if $\impbel{i} \varphi \in w  $ then $ \varphi \in v  $;

            \item for all $w  \in  W^c$ and  for all $p \in  \PROP$,
    $ w \in \valfunct^c (p) $ iff $p \in w $.

  \end{itemize}

      The next step in the proof consists in stating the following existence lemma.
    The proof is again standard  (cf. Lemma 4.20 in \cite{Bla01}) and we omit it.
   \begin{lemma} \label{existlemma}
   Let $ \varphi \in \lang$ and let $w \in W^c$.
   Then, if $ \impbelposs{i} \varphi \in w$
   then there exists $v \in\doxset^c(i,w)$
   such that $ \varphi \in v$.
   \end{lemma}

   Then, we prove the following truth lemma.
     \begin{lemma}\label{truthlemma}
Let $ \varphi \in \lang$ and let $w \in W^c$.
Then, $M^c, w \models \varphi $ iff $ \varphi \in w $.
  \end{lemma}
  \begin{proof}
  The proof is by induction on the structure of the formula.
  The cases with
  $\varphi$
  atomic, Boolean,
  and of the form
 $  \impbel{i} \psi$
  are provable in the standard way
  by means of Proposition \ref{propmcs} and Lemma  \ref{existlemma}
   (cf. Lemma 4.21 in  \cite{Bla01}).
  The proof for the case  $\varphi =  \expbel{i} \alpha$
  goes as follows: $\expbel{i} \alpha \in w$ iff
      $ \alpha \in \belbase^c(i,w) $ iff $M^c, w \models \expbel{i} \alpha $.
     \end{proof}

   The last step consists
   in proving that the canonical
   model belongs to the class $\quasinotmod $.
    \begin{proposition}\label{goodcon2}
       $M^c $ is a quasi-NDM.
    \end{proposition}
        \begin{proof}
    Thanks to Axiom (\ref{ax:KD}),
     it is easy to prove that $M^c$
     satisfies Condition C2 in Definition \ref{modeldef}.

     Let us prove that it satisfies
    Condition C1$^*$
    in Definition  \ref{QNDM}.
    To this aim, we just need to prove that
    if $\alpha \in \belbase^c(i,w) $
    then
  $   \doxset^c(i,w)  \subseteq   ||\alpha ||_{M^c}$.
  Suppose  $\alpha \in \belbase^c(i,w) $. Thus,
   $\expbel{i} \alpha \in w$.
   Hence,
by Axiom (\ref{ax:Int1})
and Proposition \ref{propmcs}, $\impbel{i} \alpha \in w$.
By the definition
of $M^c$,
it follows that, for all $v \in  \doxset^c(i,w) $,
$\alpha \in v$.
Thus,
by
Lemma
\ref{truthlemma},
for all $v \in  \doxset^c(i,w) $,
$(M^c,v) \models \alpha$.
The latter means that  $   \doxset^c(i,w)  \subseteq   ||\alpha ||_{M^c}$.
        \end{proof}

The following is
our first intermediate result.
        \begin{theorem}\label{complete1}
The logic $\logic$ is sound and complete for the class of quasi-NDMs.
\end{theorem}
\begin{proof}
As for soundness,
it is routine to check that the axioms of $\logic $
are all valid  for the class of quasi-NDMs
and that the rule of inference (\ref{ax:NecK}) preserves validity.

As for completeness, suppose that $\varphi$
   is a $\logic$-consistent formula in $\lang$. By Lemma
    \ref{lindenb}, there exists  $w \in W^c$
   such that $\varphi \in w $.
   Hence,
   by Lemma \ref{truthlemma},
    there exists  $w \in W^c$
   such that $M^c, w \models  \varphi$.
   Since, by Proposition \ref{goodcon2}, $M^c$
   is a quasi-NDM,
   we can conclude that $\varphi$
   is satisfiable for the class of quasi-NDMs.
\end{proof}

\paragraph{Soundness and completeness for NDMs and CMABs}
We can
state the two main results
of this section.
The first is about soundness and completeness
for the class of NDMs.
\begin{theorem}\label{completeness}
The logic $\logic$ is sound and complete for the class of NDMs.
\end{theorem}
\begin{proof}
It is routine exercise to verify that $\logic $
is sound for the class of NDMs.
Now, suppose that formula $\varphi$
is $\logic$-consistent. Then, by Theorems \ref{complete1} and \ref{finitemodel},  it is satisfiable for the class of  finite quasi-NDMs.
Hence, by  Theorem \ref{eqsemantics}, it is satisfiable for the class of  finite NDMs.
Thus, more generally, $\varphi$ is  satisfiable for the class of  NDMs.
\end{proof}

The second is about soundness and completeness
for the class of CMABs.
\begin{theorem}\label{completeness2}
The logic $\logic$ is sound and complete for the class of CMABs.
\end{theorem}
\begin{proofsketch}
The theorem is provable by means of Theorem \ref{completeness} and Theorem \ref{eqsemantics3}.
\end{proofsketch}

\paragraph{Decidability}
The second main result of this section is
decidability of $\logic$.
\begin{theorem}\label{decid1}
The satisfiability problem of $\logic$
is decidable.
\end{theorem}
\begin{proof}
Suppose $\varphi$ is satisfiable
for the class of NDMs.
Thus, by Theorem \ref{completeness}, it is $\logic$-consistent.
Hence, by Theorem \ref{complete1}, it is satisfiable for the class of quasi-NDMs.
From the proof of Theorem \ref{finitemodel},  we can observe that if $\varphi$
is satisfiable for the class of quasi-NDMs then there exists a quasi-NDM
 satisfying $\varphi$
such that  (i) its set of worlds contains at most $2^n$
elements, (ii) the atomic propositions outside $\PROP(\mathit{sub}(\varphi))$ are false
everywhere in the model, and
 (iii) the belief
base of an agent at a world
contains only formulas from $\mathit{sub}(\varphi)$, where $n$ is the size of
$\mathit{sub}(\varphi)$.
The construction in the proof of Theorem \ref{eqsemantics} ensures
that from this finite quasi-NDM, we can built
a finite NDM satisfying $\varphi$ for
which (i) holds and such that
(iv) the atomic propositions outside $\PROP(\mathit{sub}(\varphi)) \cup X$ are false
everywhere in the model, and (v) the belief
base of an agent at a world
contains only formulas from $\mathit{sub}(\varphi) \cup X$,
where  $X$
is an arbitrary set of atoms from $\PROP \setminus ( \PROP (\varphi) ) $
of size at most $2^n \times |\AGT|  $.
Thus,
in order to verify whether
$\varphi$ is satisfiable, we fix a $X$
and check satisfiability of $\varphi$
for all NDMs
satisfying (i), (iv) and (v).   There are finitely many
NDMs of this kind.
\end{proof}

\section{Relationship with logic of general awareness }\label{complexity}

\newcommand{  \logicaw}{\ensuremath{\textsf{LGA} }\xspace}

\newcommand{\expk}[1] {\mathtt{X}_{#1}   }
\newcommand{\awareop}[1] {\mathtt{A}_{#1}   }
\newcommand{\knowop}[1] {\mathtt{B}_{#1}  }

\newcommand{\knowoper}[1] {\mathtt{B}_{#1}  }
\newcommand{\awarefunct} {\mathcal{A}  }
\newcommand{\belrel}{\mathcal{R}  }

This section is devoted to explore the connection
between the logic $\logic$
and the logic of general awareness by  Fagin \& Halpern  (F\&H) \cite{Fag87}.
In particular, we will provide a polynomial embedding of the former
into latter and, thanks to this embedding, we will
be able to state a complexity
result
for the satisfiability
problem
of $\logic$.

The language of F\&H's logic
 of general awareness $\logicaw$, denoted
 by $\lang_\logicaw $,
 is defined by the following grammar:
 \begin{center}\begin{tabular}{lcl}
  $\varphi$  & $\bnf$ & $p  \mid \neg\varphi \mid \varphi_1 \wedge \varphi_2  \mid
  \knowop{i} \varphi \mid   \awareop{i} \varphi
               \mid       \expk{i } \varphi   $
\end{tabular}\end{center}
where $p$ ranges over $\PROP$
and $i$ ranges over $\AGT$.

The formula  $\awareop{i} \varphi$
has to be read ``agent $i$ is aware of $\varphi$''.
The operators $ \knowop{i}  $
and $\expk{i }  $ have the same interpretations
as the $\logic$ operators $\impbel{i}$
and $\expbel{i}$. Specifically,
$  \knowoper{i} \varphi$
has to be read
``agent $i$ has an implicit belief that $\varphi$ is true'', while
$  \expk{i} \varphi$
has to be read
``agent $i$ has an explicit belief that $\varphi$ is true''.

The previous language is interpreted with respect
to so-called awareness structures, that is,
tuples of the form $M = (S, \belrel_1, \ldots, \belrel_n, \awarefunct_1, \ldots, \awarefunct_n, \pi  )$
where every $\belrel_i \subseteq S \times S$
is a doxastic accessibility relation, every $\awarefunct_i : S \longrightarrow 2^{\lang_\logic }$
is an awareness function and $\pi : \ATM \longrightarrow 2^S$ is a valuation function for atomic
propositions. In order to relate $\logicaw$
with $\logic$,
we here assume that every relation $\belrel_i$
is serial to guarantee that an agent cannot have inconsistent implicit beliefs.

In $\logicaw$, the satisfaction relation
is between formulas
and pointed models $(M,s)$
where
$M = (S, \belrel_1, \ldots, \belrel_n, \awarefunct_1, \ldots, \awarefunct_n, \pi  )$
is an awareness structure and
$s \in S$ is a state:
\begin{eqnarray*}
(M,s) \models p & \Longleftrightarrow & p \in \pi (s) \\
(M,s)\models \neg \varphi & \Longleftrightarrow &    (M,s) \not \models  \varphi \\
(M,s) \models \varphi_1 \wedge \varphi_2 & \Longleftrightarrow &    (M,s)  \models \varphi_1  \text{ and }    (M,s) \models \varphi_2 \\
(M,s) \models   \knowop{i} \varphi    & \Longleftrightarrow & \forall s' \in \belrel_i(s) : (M,s') \models \varphi\\
(M,s) \models   \awareop{i} \varphi    & \Longleftrightarrow & \varphi \in \awarefunct_i(s) \\
(M,s) \models   \expk{i } \varphi     & \Longleftrightarrow & (M,s) \models \knowop{i} \varphi \text{ and } (M,s) \models \awareop{i} \varphi
\end{eqnarray*}

There are two
important differences between $\logicaw$
and $\logic$. First of all, in the semantics of $\logicaw$ the notion of doxastic alternative
is given as a primitive while
in the semantics of  $\logic$
it is defined from the concept of belief base. Secondly, the $\logicaw $
  ontology of epistemic attitudes   is richer
than the $\logic $ ontology,
as the former includes
the concept
of awareness which is not included in the latter.
We believe these are virtues of $\logic$
compared to $\logicaw$.
On the one hand,
the $\logic $
semantics offers a compact representation
of epistemic states in which the concept of belief base
plays a central role.
This
  conforms with how
epistemic states
are traditionally represented in the area
of knowledge representation and reasoning (KR).
On the other hand, modeling  explicit and implicit beliefs
without invoking the notion of awareness is a good thing,
as the latter  is intrinsically polysemic and ambiguous.
This aspect is emphasized by
 F\&H, according to whom the notion
 of awareness is  ``...open to a number of interpretations. One of them
is that an agent is aware of a formula
if he can compute whether or not it is true in a given situation within a certain time or space bound'' \cite[p. \!41]{Fag87}.

Let us define the following direct translation $\mathit{tr}$
from the $\logic$ language to the $\logicaw$
language:
\begin{align*}
\mathit{tr}(p) & = p   \text{ \ for } p \in \ATM \\
\mathit{tr}(\neg \alpha) & = \neg \mathit{tr}(\alpha) \\
\mathit{tr}( \alpha_1 \wedge \alpha_2) & =  \mathit{tr}(\alpha_1) \wedge \mathit{tr}(\alpha_2)\\
\mathit{tr}(\neg \varphi) & = \neg \mathit{tr}(\varphi) \\
\mathit{tr}( \varphi_1 \wedge \varphi_2) & =  \mathit{tr}(\varphi_1) \wedge \mathit{tr}(\varphi_2)\\
\mathit{tr}(\expbel{i} \alpha) & = \expk{i }\mathit{tr}(\alpha)  \\
\mathit{tr}(\impbel{i} \varphi) & = \knowop{i}\mathit{tr}(\varphi)
\end{align*}

As the following theorem highlights the previous
translation
provides a correct embedding
of $\logic$
into $\logicaw$.

\begin{theorem}\label{embedding}
Let $\varphi \in \lang_\logic$. Then, $\varphi $
is satisfiable for the class of quasi-NDMs if and only if $\mathit{tr}(\varphi)$
is satisfiable for the class of awareness structures.
\end{theorem}

\begin{proof}
We first prove the left-to-right direction.
Let $ M = (W , \belbase, \doxset, \valfunct )$
be a quasi-NDM
and let $w \in W$
such that $(M,w) \models \varphi$.
We  build
the corresponding structure
$M'= (S, \belrel_1, \ldots, \belrel_n, \awarefunct_1, \ldots, \awarefunct_n, \pi  )$
as follows:
\begin{itemize}
\item $S = W$,
\item for every $i \in \AGT$ and for every $w \in W$, $  \belrel_i  (w) = \doxset(i,w) $,
\item for every $i \in \AGT$ and for every $w \in W$, $  \awarefunct_i  (w) =
 \{   \varphi \in \lang_\logicaw \suchthat  \exists \alpha \in  \belbase(i,w) \text{ such that } \varphi = \mathit{tr}(\alpha)\} $,
\item for every  $w \in W$, $\pi(w) =   \valfunct (w)$.
\end{itemize}
It is easy to verify that $M'$ is an awareness structure as every relation
$ \belrel_i$ is serial.

By induction on the structure of $\varphi$,
we prove that
 for all $w \in W$,
  ``$(M, w) \models \varphi$
  iff $(M',w) \models   \mathit{tr}(\varphi)$''.

The case $\varphi = p $ and the boolean cases $\varphi = \neg \psi$
and $\varphi = \psi_1 \wedge \psi_2$
are clear.
Let us now consider the case $\varphi = \expbel{i }\alpha$.

($\Rightarrow$)
$(M, w) \models \expbel{i }\alpha$
means that $\alpha \in  \belbase(i,w)$. By definition of $\awarefunct_i $,
the latter implies that $\mathit{tr}(\alpha) \in   \awarefunct_i  (w) $
which is equivalent to $(M',w) \models  \awareop{i} \mathit{tr}(\alpha)   $.
Moreover, $(M, w) \models \expbel{i }\alpha$ implies that
$ \doxset(i,w) \subseteq || \alpha ||_M $.
By induction hypothesis, we have
$|| \alpha ||_M = || \mathit{tr}(\alpha)  ||_{M'} $.
Thus, by definition of $\belrel_i  (w)$,
it follows that $\belrel_i  (w)  \subseteq   || \mathit{tr}(\alpha)  ||_{M'}$.
The latter means that $(M',w) \models   \knowop{i} \mathit{tr}(\alpha)   $.
From the latter and
$(M',w) \models  \awareop{i} \mathit{tr}(\alpha)   $,
it follows that $(M',w) \models   \expk{i } \mathit{tr}(\alpha)   $.

($\Leftarrow$)
 $(M',w) \models   \expk{i } \mathit{tr}(\alpha)   $
implies  $(M',w) \models    \awareop{i} \mathit{tr}(\alpha)   $
which is equivalent to $\mathit{tr}(\alpha) \in   \awarefunct_i  (w) $.
By definition of $  \awarefunct_i$,
the latter implies
 $\alpha \in    \belbase(i,w) $
 which is equivalent to
  $(M,w) \models    \expbel{i} \alpha   $.

  Finally, let us consider the case $\varphi = \impbel{i }\psi$.
  By induction hypothesis, we have
  $|| \psi ||_M = || \mathit{tr}(\psi)  ||_{M'} $.
    $(M, w) \models \impbel{i }\psi$ means that
  $\doxset(i,w) \subseteq || \psi ||_M$.
  By definition of $\belrel_i (w) $ and   $|| \psi ||_M = || \mathit{tr}(\psi)  ||_{M'} $,
  the latter is equivalent
  to $\belrel_i  (w)  \subseteq  || \mathit{tr}(\psi)  ||_{M'}$
  which in turn is equivalent to $(M', w) \models \knowop{i} \mathit{tr}(\psi)$.

Let us prove the right-to-left direction.
Let
$M= (S, \belrel_1, \ldots, \belrel_n, \awarefunct_1, \ldots, \awarefunct_n, \pi  )$
be an awareness structure.
We build the model
$ M'= (W , \belbase, \doxset, \valfunct )$
as follows:
\begin{itemize}
\item $W = S$,

\item for every $i \in \AGT$ and for every $s \in S$, $   \doxset(i,s) =  \belrel_i(s)  $,

\item for every $i \in \AGT$ and for every $s \in S$, $  \belbase(i,s) = \{ \alpha \in \lang_\logic \suchthat
 \belrel_i(s) \subseteq || \mathit{tr}(\alpha) ||_M \text{ and } \mathit{tr}(\alpha) \in \awarefunct_i(s)
 \} $,

\item for every  $s \in S$, $   \valfunct (s)= \pi(s)$.
\end{itemize}
Let us prove that $M'$
is a quasi-NDM.
It clearly satisfies Condition C2 in Definition \ref{modeldef}.
In order to prove that it satisfies Condition C1$^*$  in Definition \ref{QNDM}, we first prove
by induction on the structure of $\alpha$
 that
  $|| \mathit{tr}(\alpha)   ||_{M}  = || \alpha  ||_{M'} $.
  The case $\alpha= p$
  is clear as well as the boolean
  cases.
Let us prove the case $\alpha= \expbel{i}\alpha'$.
We have $M', s \models \expbel{i}\alpha'$ iff $\alpha ' \in  \belbase(i,s) $.
By definition of $\belbase(i,s)$,
we have $\alpha ' \in  \belbase(i,s) $ iff
$\belrel_i(s) \subseteq || \mathit{tr}(\alpha') ||_M$
and $ \mathit{tr}(\alpha') \in \awarefunct_i(s) $.
The latter is equivalent to
 $M, s \models \expk{i }\mathit{tr}(\alpha')$.

Suppose that $\alpha \in \belbase(i,s)$.
By definition of $ \belbase(i,s)$,
it follows that $ \belrel_i(s) \subseteq || \mathit{tr}(\alpha) ||_M$.
Thus,
since   $|| \mathit{tr}(\alpha)   ||_{M}  = || \alpha  ||_{M'} $,
we have
$ \belrel_i(s) \subseteq  || \alpha  ||_{M'} $.
Hence, by definition of $\doxset(i,s)$,
$ \doxset(i,s) \subseteq  || \alpha  ||_{M'} $.
This shows that $M'$ satisfies Condition C1$^*$.

In the rest of the proof we show that
 for all $s \in S$,
  ``$(M, s) \models \mathit{tr}(\varphi)$
  iff $(M',s) \models   \varphi$''.
  The proof is by induction on the structure of $\varphi$.

The case $\varphi = p $ and the boolean cases
are clear.
Let us now consider the case $\varphi = \expbel{i }\alpha$.
$(M, w) \models  \mathit{tr}( \expbel{i }\alpha)$
means that
$(M, w) \models   \expk{i }\mathit{tr}(\alpha) $.
The latter is equivalent to $\belrel_i(s) \subseteq || \mathit{tr}(\alpha) ||_M$
and $ \mathit{tr}(\alpha) \in \awarefunct_i(s) $.
By definition of $    \belbase(i,s)$,
the latter is equivalent to
$\alpha \in   \belbase(i,s)$
which in turn is equivalent to $(M', w) \models \expbel{i }\alpha $.

Let us finally consider the case $\varphi = \impbel{i }\psi$.
  By induction hypothesis, we have
  $|| \mathit{tr}(\psi) ||_M = || \psi  ||_{M'} $.
    $(M, w) \models  \mathit{tr}(\impbel{i }\psi)$ means that
     $(M, w) \models    \knowop{i}\mathit{tr}(\psi)  $
     which in turn means that $\belrel_i (w)  \subseteq || \mathit{tr}(\psi)  ||_M $.
  By definition of $\doxset(i,w)$ and   $|| \psi ||_M = || \mathit{tr}(\psi)  ||_{M'} $,
  the latter is equivalent
  to $\doxset(i,w)  \subseteq  || \psi  ||_{M'}$
  which in turn is equivalent to $(M', w) \models  \impbel{i }\psi$.
\end{proof}

The following theorem is a direct consequence of 
Theorems \ref{finitemodel}, \ref{eqsemantics}, \ref{eqsemantics3} 	and \ref{embedding}.
\begin{theorem}
Let $\varphi \in \lang_\logic$. Then, $\varphi $
is satisfiable for the class of CMABs  if and only if $\mathit{tr}(\varphi)$
is satisfiable for the class of awareness structures.
\end{theorem}

In \cite{AlechinaAgotnes2},
it is proved that,
for every $X \subseteq \{\text{reflexivity}, \text{transitivity}, \text{Euclideanity} \}$,
the satisfiability problem
for the logic
of general awareness
interpreted over 
awareness structures
whose relations $\belrel_i$
satisfy all properties in $X$
is PSPACE-complete.
To prove PSPACE-membership,
an adaptation of the tableau method for
multi-agent
epistemic logic by \cite{DBLP:journals/ai/HalpernM92} is proposed.
PSPACE-hardness follows from the PSPACE-hardness of multi-agent epistemic logics
proved by \cite{DBLP:journals/ai/HalpernM92}.
It is easy to adapt {\AA}gotnes \& Alechina's method
to show that  the
 logic
of general awareness
interpreted over 
awareness structures
whose relations $\belrel_i$
are serial 
is also PSPACE-complete.
As a consequence, we can prove the following result.

\begin{theorem}
The satisfiability problem of $\logic$
is PSPACE-complete.
\end{theorem}

\begin{proofsketch}
PSPACE-membership 
follows from the previous polynomial-time reduction
of $\logic$ satisfiability problem 
to $\logicaw$ satisfiability problem
relative to
serial awareness structures
and the fact that the latter problem is in PSPACE.
PSPACE-hardness follows from the PSPACE-hardness of multi-agent epistemic logics \cite{DBLP:journals/ai/HalpernM92}.
\end{proofsketch}

%
%
%
%
%
%
%
%
%
%
%
%
%
%
%
%
%
%

\section{Related work}\label{relwork}

The present work
lies in the area of logics for non-omniscient agents.
Purely syntactic
approaches
to the logical omniscience problem have been proposed in which
 an agent's beliefs are described either by a set of formulas
 which is
 not necessarily closed under deduction \cite{Eberle,MooreHendrix}
 or by a set of formulas
 obtained by the application
 of an incomplete set of deduction rules \cite{KONOLIGE}.
Logics
of time-bounded reasoning
have also been studied \cite{DBLP:conf/atal/AlechinaLW04,DBLP:journals/aamas/GrantKP00},
in which
reasoning is a represented as a process that requires time
due to the time-consuming application of inference rules.
Finally,
logics of (un)awareness
have been studied  both in AI \cite{Fag87,DBLP:journals/jolli/DitmarschF14,DBLP:journals/jolli/AgotnesA14a}
and economics
 \cite{Modica,Heifetz,DBLP:journals/geb/HalpernR09}.
 
 As we have shown in Section \ref{complexity},
 our logic $\logic$ is closely related to Fagin \& Halpern (F\&H)'s
 logic of general awareness,
 as there exists a polynomial embedding
 of the former into the latter.
Another related system  
 is the logic of local reasoning also presented in \cite{Fag87}
 in which the distinction between explicit and implicit
 beliefs is captured.
 F\&H) use a neighborhood semantics
 for explicit belief:
 every agent is associated with a set of sets
 of worlds, called frames of mind.
 They define an agent's set of
 doxastic alternatives as the intersection
 of the agent's frames of mind.
 According to F\&H's semantics, an agent
  explicitly believes that
 $\varphi$
 if and only if she
 has a frame of mind in which $\varphi$
 is globally true. Moreover, an agent
  implicitly believes
 that $\varphi$
 if and only if, $\varphi$
 is true at all her doxastic alternatives.
In their semantics,
 there is no representation of an agent's belief
 base, corresponding to the set of formulas explicitly believed by the agent.
 Moreover, differently from our notion
 of explicit belief, their notion does not completely solve
 the logical omniscience problem.
 For instance, while their notion of explicit
 belief is closed under logical equivalence,
our notion is not.
Specifically, the following rule of equivalence
preserves validity in F\&H's logic
but not in our logic:
\begin{align*}
\frac{\alpha \leftrightarrow \alpha'}{  \expbel{i} \alpha \leftrightarrow   \expbel{i} \alpha'}
\end{align*}
This is a consequence
of their use of an extensional semantics
for explicit belief.
Levesque too provides an extensional semantics
for explicit belief with no connection
with the notion of belief base  \cite{LevesqueExplicitBel}. In his logic,
explicit beliefs are closed under conjunction,
while they are not in our logic $\logic$.

\section{Conclusion}\label{conclusion}

We
have presented
a logic
of explicit and implicit beliefs
with a semantics
based on belief bases.
In the future,
we plan to study a variant
of this logic
in which explicit and implicit beliefs
are replaced by
\emph{truthful}
explicit and implicit knowledge.
At the semantic level, we will move
from
multi-agent belief
bases
to multi-agent knowledge bases
in which the epistemic accessibility relation $ \relstate{i}$
is assumed to be reflexive. The logic will include the following extra-axiom:
\begin{align}
&  \impbel{i} \varphi \rightarrow \varphi \tagLabel{T$_{\impbel{i}}$}{ax:KT}
\end{align}

We also expect to study a variant of our logic
with the following extra-axioms
of
positive and negative introspection for implicit beliefs:
 \begin{align}
&   \impbel{i} \varphi \rightarrow \impbel{i} \impbel{i} \varphi \tagLabel{4$_{\impbel{i}}$}{ax:K4}\\
&  \neg \impbel{i} \varphi \rightarrow \impbel{i} \neg \impbel{i} \varphi \tagLabel{5$_{\impbel{i}}$}{ax:K5}
\end{align}
Moreover,
we plan to extend
the logic $\logic$
and its epistemic variant
by concepts
of distributed belief and distribute knowledge.

We also plan to study a dynamic extension of the logic 
$\logic$
which is 
in line with existing theories of belief base change.
Specifically, we intend
to capture different forms
of belief base revision
in the multi-agent setting
offered by $\logic$
including 
partial meet base revision
\`a la Hansson  \cite{HanssonPHD}.
The general idea
of partial meet base revision,
inspired by single-agent
partial meet revision
for belief sets   \cite{Alc85},
is that the belief
base resulting from the integration
of an input formula $\alpha$
should be equal to the intersection
of all maximally consistent subsets
of the initial belief base  including $\alpha$.

Last but not least,
we plan to study the model checking problem for our logic
by using
the compact representation offered by multi-agent belief models,
as defined in Definition \ref{MAM}.
As shown by \cite{BenthemEijckGattingerSu2015}, the possibility of using compact models could be beneficial for model checking.

\end{document}